\documentclass[twoside,11pt]{article}
\usepackage{jmlr2e}
\usepackage{amsmath}

\let\epsilon\varepsilon
\let\phi\varphi

\def\X{{\cal X}}
\def\N{\mathbb N}

\def\C{\mathcal C}
\def\S{\mathcal S}
\def\E{{\bf E}}
\def\v{{v}}

\def\-as{\text{-a.s.}}

\def\P{\mathcal P}

\def\geklz{\ge_{\scriptscriptstyle KL}}

\def\getv{\ge_{\scriptscriptstyle tv}}

\begin{document}
\title{On the Relation between Realizable and Nonrealizable Cases of the Sequence Prediction Problem}
\ShortHeadings{Realizable and Nonrealizable Prediction Problems}{Daniil Ryabko}
\author{\name Daniil Ryabko \email   daniil@ryabko.net \\ \addr  INRIA Lille-Nord Europe,\\ 40, avenue Halley,\\
Parc Scientifique de la Haute Borne\\
59650 Villeneuve d'Ascq, France
}

\jmlrheading{12}{2011}{1-21}{12/10}{6/11}{Daniil Ryabko}

\firstpageno{1}

\editor{Nicol\`o Cesa-Bianchi}
\maketitle

\begin{abstract}
A sequence $x_1,\dots,x_n,\dots$ of discrete-valued observations is generated 
according to some unknown probabilistic law (measure) $\mu$. 
After observing each outcome, 
one is required to give  conditional probabilities of the next observation.
The realizable case is when the  measure  $\mu$ belongs to an arbitrary but known class $\C$  of  process measures.
The non-realizable case is when $\mu$ is completely arbitrary, but the prediction performance is measured
with respect to a given set $\C$ of process measures.
We are interested in the relations between these problems and between their solutions, as well as 
in characterizing the cases when a solution exists and finding these solutions.
We show that if the quality of prediction is measured using the total variation distance, then these problems coincide,
while if it is measured using the expected average KL divergence, then they are different. For some of the formalizations
we also show that when a solution exists, it can be obtained as a Bayes mixture over a countable subset of $\C$.
We also obtain several characterization of those sets $\C$ for which solutions to the considered problems exist.
As an illustration to the general results obtained, we show that a solution to the non-realizable case of the 
sequence prediction problem exists for the set of all finite-memory processes, but does not exist for the set
of all stationary processes. 
 It should be emphasized that the framework  is completely general: the  processes measures considered are not required to be 
 i.i.d., mixing, stationary, or to belong to any parametric family.
\end{abstract}


\section{Introduction}
A sequence $x_1,\dots,x_n,\dots$ of discrete-valued observations (where $x_i$ belong to a finite set  $\X$) is generated 
according to some unknown probabilistic law (measure). That is, $\mu$ is a probability 
measure on the space $\Omega=(\X^\infty,\mathcal B)$ of one-way infinite sequences  (here $\mathcal B$ is the usual Borel $\sigma$-algebra).
After each new outcome $x_n$ is revealed, one is required to predict conditional {\em probabilities} of the next observation $x_{n+1}=a$, $a\in \X$, 
given the past $x_1,\dots, x_n$.
Since a predictor $\rho$  is required to give conditional probabilities $\rho(x_{n+1}=a|x_1,\dots,x_n)$ for all possible 
histories $x_1,\dots, x_n$, it defines itself a probability measure on the space  $\Omega$ of one-way infinite sequences.
In other words, a probability measure on $\Omega$ can be considered both as a data-generating mechanism and as a predictor.

Therefore, given a set $\mathcal C$ of probability measures on $\Omega$, one can ask two kinds of questions about $\mathcal C$.
First, does there exist a predictor $\rho$, whose forecast probabilities converge (in a certain sense) to the $\mu$-conditional
probabilities, if an arbitrary $\mu\in\C$ is chosen to generate the data?  Here we assume that the ``true'' measure that generates
the data belongs to the set $\C$ of interest, and would like to construct a predictor that predicts all measures in $\C$.
The second type of questions is as follows: does there exist a predictor that predicts at  least as well as any predictor $\rho\in\C$, 
if the measure that generates the data comes possibly from outside of $\C$? Thus, here we consider elements of $\C$ as predictors, 
and we would like to combine their predictive properties, if this is possible.
Note that in this setting the two questions above concern the same object: a set $\C$ of probability measures on $\Omega$. 

Each of these two questions,  the realizable and the non-realizable one,  have enjoyed much attention in the literature; 
the setting for the non-realizable case is usually slightly different, which is probably why it has not (to the best
of the author's knowledge) been studied as another facet of the realizable case. 
The realizable case traces back to Laplace, who has considered the problem of predicting outcomes of a series 
of independent tosses of a biased coin. That is, he has considered the case when the set $\C$ is that of all 
i.i.d.\ process measures. Other classical examples studied are the set of all computable (or semi-computable) 
measures \citep{Solomonoff:78}, the set of $k$-order Markov and finite-memory processes (e.g., \citealp{Krichevsky:93})
and the set of all stationary processes \citep{BRyabko:88}. The general question of finding predictors for an arbitrary
given set $\C$ of process measures has been addressed in \citep{Ryabko:07pqisit, Ryabko:08pqaml, Ryabko:10pq3+}; the latter work 
shows that when a solution exists it can be obtained as a Bayes mixture over a countable subset of~$\C$.

The non-realizable case is usually studied in a slightly different, non-probabilistic, setting. We 
refer to \citep{Cesa:06} for a comprehensive overview. It is usually assumed 
that the observed sequence of outcomes is an arbitrary (deterministic) sequence; it is  required not to give
conditional probabilities, but just deterministic guesses (although these guesses can be selected using randomisation). Predictions result in a certain loss, which 
is required to be small as compared to the loss of a given set of reference predictors (experts) $\C$. The losses of the experts
and the predictor are observed after each round.
 In this approach, it is mostly assumed
that the set $\C$ is finite or countable. 
The main difference with the formulation considered in this work is that we require a predictor to give probabilities, and 
thus the loss is with respect to something never observed (probabilities, not outcomes). The loss itself is not completely 
observable in our setting.  In this sense our non-realizable
version of the problem is more  difficult. 
Assuming that the data generating mechanism is probabilistic, 
even if it is completely unknown, makes sense in such problems as, for example, game playing, or market analysis. 
In these cases one may wish to assign smaller loss to those models or experts who give probabilities closer 
to the correct ones (which are never observed), even though different probability forecasts can often result
in the same action. 
Aiming at predicting probabilities of outcomes  also allows us to abstract
from the actual use of the predictions (for example, making bets) and thus from considering losses in a general form; instead,
we can concentrate on those forms of loss  that
are more convenient for the analysis. In this latter respect, the problems we consider are easier than those considered in prediction with expert advice.
(However, in principle, nothing restricts us to considering the simple losses that we chose; they are just a convenient choice.)
Noteworthy, the probabilistic approach also  makes the machinery of probability theory applicable, hopefully making the problem easier.
A reviewer suggested the following summary explanation of the difference between the non-realizable problems of this work and 
prediction with expert advice: the latter is prequential (in the sense of \citealp{Dawid:92}), whereas the former is not.

In this work we consider two measures of the quality of prediction. The first one is the total variation distance, which 
measures the difference between the forecast and the ``true'' conditional probabilities of all future events (not just the probability of the next
outcome). The second one is expected (over the data) average (over time) Kullback-Leibler divergence. 
Requiring that  predicted and true probabilities converge in total variation is very strong; in particular, this is possible
if \citep{Blackwell:62} and only if \citep{Kalai:94} the process measure generating the data is absolutely continuous with 
respect to the predictor. The latter fact makes the sequence prediction problem relatively easy to analyse. Here  we investigate 
 what can be paralleled for the other measure of prediction quality (average KL divergence), which is much weaker, 
and thus allows for solutions for the cases of much larger sets $\C$ of process measures (considered either as predictors
or as data generating mechanisms).

Having introduced our measures of prediction quality, we can further break the non-realizable case into two problems.
The first one is as follows. Given  a set $\C$ of predictors, we want to find a predictor whose prediction error converges to zero 
if there is at least one predictor in $\C$ whose prediction error converges to zero; we call this problem simply the ``non-realizable'' case, or Problem~2 (leaving
the name ``Problem~1'' to the realizable case). 
The second non-realizable problem  is the ``fully agnostic'' problem: it is to make the prediction error asymptotically as small as that of the best 
(for the given process measure generating the data)
 predictor in $\C$ (we call this Problem~3). Thus, we now have three problems about a set of process measures $\C$ to address.

We show  that if the quality of prediction is measured in total variation, then all the three problems coincide: any 
solution to any one of them is a solution to the other two. For the case of expected average KL divergence, all the three problems
are different: the realizable  case is strictly easier than non-realizable (Problem 2), which is, in turn, strictly easier
than the fully agnostic case (Problem~3). We then analyse which results concerning prediction in total variation can be transferred to which of the problems
concerning prediction in average KL divergence. It was shown in \citep{Ryabko:10pq3+} that, for the realizable case, 
if there is a solution for a given set of process measures $\C$, then a solution can also be obtained 
as a Bayesian mixture over a countable subset of $\C$; this holds both for prediction in total variation and in 
expected average KL divergence. Here we show that this result also holds true for the (non-realizable) case of Problem~2, 
for prediction in expected average KL divergence. We do not have an analogous result for Problem~3 (and, in fact, conjecture that the opposite statement holds true).
However, for the fully agnostic case of Problem~3, we show that separability with respect to a certain 
topology given by KL divergence is a sufficient (though not  a necessary) condition for the existence of 
a predictor. This is used to demonstrate that there is a solution to this problem for the set of all finite-memory 
process measures, complementing similar results obtained earlier in different settings. On the other hand, we show that there is no solution to this problem for the set of all stationary process
measures, in contrast to a result of B.~\cite{BRyabko:88} that gives a solution to the realizable case of this problem (that is,
a predictor whose expected average KL error goes to zero if any stationary process is chosen to generate the data).
Finally, we also consider a modified version of Problem~3, in which the performance of predictors is only 
compared on individual sequences. For this problem, we obtain, using a result from \citep{BRyabko:86}, a characterisation of those sets $\C$ for which a solution exists in terms
of the Hausdorff dimension.

\section{Notation and Definitions}\label{s:pre}
Let $\X$ be a finite set. The notation $x_{1..n}$ is used for $x_1,\dots,x_n$. 
 We consider  stochastic processes (probability measures) on $\Omega:=(\X^\infty,\mathcal B)$ where $\mathcal B$
is the sigma-field generated by the cylinder sets  $[x_{1..n}]$, $x_i\in\X, n\in\N$ 
($[x_{1..n}]$ is the set of all infinite sequences that start with $x_{1..n}$).
For a  finite set $A$ denote $|A|$ its cardinality.
We use  $\E_\mu$ for
expectation with respect to a measure $\mu$.

Next we introduce the measures of the quality of prediction used in this paper.
For two measures  $\mu$ and $\rho$  we are interested in how different 
 the $\mu$- and $\rho$-conditional probabilities are, given a data sample $x_{1..n}$.
Introduce the {\em (conditional) total variation} distance 
$$
v(\mu,\rho,x_{1..n}):= \sup_{A\in\mathcal B} |\mu(A|x_{1..n})-\rho(A|x_{1..n})|,
$$
if $\mu(x_{1..n})\ne0$ and $\rho(x_{1..n})\ne0$, and $v(\mu,\rho,x_{1..n})=1$ otherwise.
\begin{definition}
We say that $\rho$ predicts $\mu$ in total variation if 
$$
v(\mu,\rho,x_{1..n})\to0\ \mu\-as
$$
\end{definition}
This convergence is rather strong. In particular, it means that $\rho$-conditional probabilities
of arbitrary far-off events converge to $\mu$-conditional probabilities. 
Moreover,  $\rho$ predicts $\mu$ in
total variation if \citep{Blackwell:62} and only if \citep{Kalai:94} $\mu$ is absolutely continuous with respect to $\rho$.
Denote $\getv$ the relation of absolute continuity (that is, $\rho\getv\mu$ if $\mu$ is absolutely continuous with respect to $\rho$).

Thus, for a class $\C$ of measures there is a predictor $\rho$ that predicts every $\mu\in\C$ in total
variation if and only if every $\mu\in\C$ has a density with respect to $\rho$.
Although such  sets of processes are rather large, they do not include even such basic 
examples as the set of all Bernoulli i.i.d.\ processes.
That is, there is no $\rho$ that would predict in total variation every Bernoulli i.i.d.\ process measure $\delta_p$, $p\in[0,1]$,
where $p$ is the probability of $0$. 
Indeed, all these processes $\delta_p$, $p\in[0,1]$, are singular with respect to one another; in particular, 
each of the non-overlapping sets $T_p$ of all sequences which have limiting fraction $p$ of 0s has probability 1 with respect to one of the measures and 
0 with respect to all others; since there are uncountably many of these measures, there is no measure $\rho$ with respect to which they all would have a density 
(since such a measure should have $\rho(T_p)>0$ for all $p$).

Therefore, perhaps for many (if not most) practical applications this measure of the quality of prediction is too strong,
and one is interested in weaker measures of performance.

For two measures $\mu$ and $\rho$ introduce the {\em expected cumulative Kullback-Leibler divergence (KL divergence)} as
\begin{equation}\label{eq:akl} 
  d_n(\mu,\rho):=  \E_\mu
  \sum_{t=1}^n  \sum_{a\in\X} \mu(x_{t}=a|x_{1..t-1}) \log \frac{\mu(x_{t}=a|x_{1..t-1})}{\rho(x_{t}=a|x_{1..t-1})},
\end{equation}
In words, we take the expected (over data) cumulative (over time) KL divergence between $\mu$- and $\rho$-conditional (on the past data) 
probability distributions of the next outcome.
\begin{definition}
We say that $\rho$ predicts $\mu$ in expected average KL divergence if 
$$
{1\over n} d_n(\mu,\rho)\to0.
$$
\end{definition}
This measure of performance is much weaker, in the sense that it requires good predictions only one step ahead, and not on every step
but only on average; also the convergence is not with probability 1 but in expectation. With prediction quality so measured, 
predictors  exist for relatively large
classes of measures; most notably, \cite{BRyabko:88} provides a predictor which predicts every stationary 
process in expected average KL divergence. 

We will use the following well-known identity  (introduced, in the context of sequence prediction,  by \citealp{BRyabko:88})
\begin{equation}\label{eq:kl}
 d_n(\mu,\rho)=-\sum_{x_{1..n}\in\X^n}\mu(x_{1..n}) \log \frac{\rho(x_{1..n})}{\mu(x_{1..n})},
\end{equation}
where on the right-hand side we have simply the KL divergence between measures $\mu$ and $\rho$ restricted to the first $n$ observations.

Thus, the results of this work will be established with respect to two very different measures of prediction quality,
one of which is very strong and the other rather weak. This suggests that the facts established reflect some fundamental
properties of the problem of prediction, rather than those pertinent to particular measures of performance. On the other hand,
it remains open to extend the results below to different measures of performance.

\begin{definition} Consider the following classes of process measures: $\mathcal P$ is the set of all process measures, 
 $\mathcal D$  is the set of all degenerate discrete process measures, $\mathcal S$ is the set of all stationary processes and
$\mathcal M_k$ is the set of all stationary measures with memory not greater than $k$ ($k$-order Markov processes, with $\mathcal M_0$ being
the set of all i.i.d.\ processes):
\begin{equation}
\mathcal D:=\left\{\mu\in\mathcal P: \exists  x\in\X^\infty \ \ \mu(x)=1\right\},
\end{equation}
\begin{equation}
\mathcal S:=\left\{\mu\in\mathcal P: \forall n,k\ge1\,\forall a_{1..n}\in\X^n\, \mu(x_{1..n}=a_{1..n})=\mu(x_{1+k..n+k}=a_{1..n})\right\}.
\end{equation}
\begin{multline}\label{eq:mk}
\mathcal M_k:=\left\{\mu\in\mathcal S:\forall n\ge k\,\forall a\in\X\, \forall a_{1..n}\in\X^n \right. \\ \left. \mu(x_{n+1}=a|x_{1..n}=a_{1..n})=\mu(x_{k+1}=a|x_{1..k}=a_{n-k+1..n})\right\}.
\end{multline}
\end{definition}
Abusing the notation, we will sometimes use elements of $\mathcal D$ and $\X^\infty$ interchangeably.
The following (simple and well-known)  statement   will be used repeatedly in the examples.
\begin{lemma}\label{th:disc}
 For every $\rho\in\mathcal P$ there exists $\mu\in\mathcal D$ such that $d_n(\mu,\rho)\ge n\log|\X|$ for all $n\in\N$.
\end{lemma}
\begin{proof}
 Indeed, for each $n$ we can select $\mu(x_n=a)=1$ for $a\in\X$ that minimizes $\rho(x_n=a|x_{1..n-1})$, so that $\rho(x_{1..n})\le|\X|^{-n}$.
\end{proof}

\section{Sequence Prediction Problems}\label{s:ba}
For the two notions of predictive quality introduced, we can now  state formally the sequence prediction problems.
\\\noindent{\bf Problem 1}(realizable case). 
Given a set of probability measures $\C$, find a measure $\rho$ such that $\rho$  predicts in total variation (expected average KL divergence) every $\mu\in\C$, 
if such a $\rho$ exists.

Thus, Problem~1 is about finding a predictor for the case when the process generating the data is known to belong to 
a given class $\mathcal C$. That is, the set $\C$ here is a set of measures that generate the data. Next let us formulate
the questions about $\C$ as a set of predictors.

\noindent{\bf Problem 2} (non-realizable case). Given a set of process measures (predictors) $\C$, find a process measure $\rho$ such that $\rho$ predicts in total 
variation (in expected average KL divergence) every measure $\nu\in\mathcal P$ such  that there is $\mu\in\C$ which 
predicts  (in the same sense) $\nu$.

While Problem~2 is already quite general, it does not yet address what can be called the fully agnostic case:  if nothing
at all is known about the process $\nu$ generating the data, it means that there may be no $\mu\in\C$ such that $\mu$ predicts $\nu$, 
and then, even if we have a solution $\rho$ to the Problem~2, we still do not know what the performance of $\rho$  is going
to be on the data generated by $\nu$, compared to the performance of the predictors from $\C$. To address this  fully agnostic case we have to introduce 
the notion of loss.

\begin{definition}\label{def:loss} Introduce the  almost sure total variation loss of $\rho$ with respect to~$\mu$
$$
l_{tv}(\mu,\rho):=\inf\{\alpha\in[0,1]: \limsup_{n\to\infty} v(\mu,\rho,x_{1..n})\le\alpha\ \mu\text{--a.s.}\},
$$
and the asymptotic KL loss
$$
l_{KL}(\nu,\rho):=\limsup_{n\to\infty}{1\over n}d_n(\nu,\rho).
$$
\end{definition}

We can now formulate the fully agnostic version of the sequence prediction problem.

\noindent{\bf Problem 3.} Given a set of process measures (predictors) $\C$, find a process measure $\rho$ such that $\rho$ predicts
at least as well as any $\mu$ in $\C$, if any process measure $\nu\in\mathcal P$ is chosen to generate the data: 
\begin{equation}\label{eq:p3}
l(\nu,\rho) - l(\nu,\mu)\le0 
\end{equation}
 for every $\nu\in\mathcal P$ and every $\mu\in\C$, where $l(\cdot,\cdot)$ is either $l_{tv}(\cdot,\cdot)$ or $l_{KL}(\cdot,\cdot)$.

The three problems just formulated represent different conceptual approaches to the sequence prediction problem.
Let us illustrate the difference by the following {\bf informal example}. Suppose that the set $\C$ is that of all 
(ergodic, finite-state) Markov chains. Markov chains being a familiar object in probability and statistics, 
we can  easily construct a predictor $\rho$ that predicts every $\mu\in\C$ (for example, in expected
average KL divergence, see \citealp{Krichevsky:93}). That is, if we know that the process $\mu$ generating the data is Markovian, 
we know that our predictor is going to perform well. This is the realizable case of Problem~1. In reality,  
rarely can we be sure that the Markov assumption holds true for the data at hand. We may believe, however, 
that it is still a reasonable assumption, in the sense that there is a Markovian model 
which, for our purposes (for the purposes of prediction), is a good model of the data. 
Thus we may assume that there is a Markov model (a predictor) that predicts well the process 
that we observe, and we would like to combine the predictive qualities of all these Markov models.
This is the ``non-realizable'' case of Problem~2.
Note that this problem is more difficult than the first one; in particular, a process $\nu$ generating
the data may be singular with respect to any Markov process, and still be  predicted well (in the sense of expected average
KL divergence, for example) by some of them.  Still, here we are making some assumptions about the process generating 
the data, and, if these assumptions are wrong, then we do not know anything about the performance of our predictor.
Thus, we may ultimately wish to acknowledge that we do not know anything at all about the data; we still know 
a lot about Markov processes, and we would like to use this knowledge on our data. If there is anything at all 
Markovian in it (that is, anything that can be captured by a Markov model), then we would like our predictor to use
it. In other words, we want to have  a predictor that predicts any process measure whatsoever (at least) as well as any Markov predictor.
This is the ``fully agnostic'' case of Problem~3.

Of course, Markov processes were just mentioned as an example, while in this work we are only concerned with the most general
case of arbitrary  (uncountable) sets $\C$ of process measures.

 The following statement is rather obvious.
 \begin{proposition}\label{th:1}
  Any solution to Problem~3 is a solution to Problem~2, and any solution to Problem~2 is a solution to Problem~1.
 \end{proposition}
Despite the conceptual differences in formulations, it may be somewhat unclear whether the three problems are indeed different. 
It appears that this depends on the measure of predictive quality chosen: for the case of prediction in total variation distance 
all the three problems coincide, while for the case of prediction in expected average KL divergence they are different.

\section{Prediction in Total Variation}
As it was mentioned,   a measure $\mu$ is absolutely continuous with respect to a measure $\rho$
if and only if $\rho$ predicts $\mu$ in total variation distance. 
 This reduces  studying at least  Problem~1 for total variation distance to studying the relation of absolute
continuity. Introduce the notation $\rho\getv\mu$ for this relation. 

Let us briefly  recall some facts we know about $\getv$; details can be found, for example, in \citep{Plesner:46}. 
Let $[\P]_{tv}$ denote the set of equivalence classes of $\P$ with respect
to $\getv$, and for $\mu\in  \P_{tv}$ denote $[\mu]$ the equivalence class that contains $\mu$.
 Two elements $\sigma_1,\sigma_2\in[\P]_{tv}$ (or  $\sigma_1,\sigma_2\in\P$) are called disjoint (or singular) if there is no $\nu\in [\P]_{tv}$ such that 
$\sigma_1\getv\nu$ and $\sigma_2\getv\nu$; in this case we write $\sigma_1\perp_{tv}\sigma_2$. 
We write $[\mu_1] +[\mu_2]$ for  $[{1\over2}(\mu_1 + \mu_2)]$.
Every pair $\sigma_1,\sigma_2\in[\P]_{tv}$ has a supremum $\sup(\sigma_1,\sigma_2)=\sigma_1+\sigma_2$.
Introducing into $[\P]_{tv}$ an extra element $0$ such that $\sigma\getv0$ for all $\sigma\in[\P]_{tv}$, 
we can state that for every $\rho,\mu\in  [\P]_{tv}$ there exists a unique pair of elements $\mu_s$ and $\mu_a$
such that $\mu= \mu_a + \mu_s$, $\rho\ge\mu_a$ and $\rho\perp_{tv}\mu_s$. (This is a form of Lebesgue decomposition.) Moreover, $\mu_a=\inf (\rho,\mu)$. 
Thus, every pair of elements has a supremum and an infimum. 
 Moreover, every bounded set of disjoint elements of $[\P]_{tv}$ is at most countable.

Furthermore, we  introduce the (unconditional) total variation distance between process measures.
\begin{definition}[unconditional total variation distance]
 The  (unconditional) total variation distance is defined as 
$$
v(\mu,\rho):= \sup_{A\in\mathcal B} |\mu(A)-\rho(A)|.
$$
\end{definition}

Known characterizations of  those sets $\C$ that are bounded with respect to $\getv$ can now be related to  our prediction problems 1-3 as follows.
\begin{theorem}\label{th:tv} Let  $\mathcal C\subset\mathcal P$. The following statements about  $\C$ are equivalent.
\begin{itemize}
  \item[(i)] There exists a solution to Problem 1 in total variation.
  \item[(ii)] There exists a solution to Problem 2 in total variation.
  \item[(iii)] There exists a solution to Problem 3 in total variation.
 \item[(iv)]  $\C$ is upper-bounded with respect to $\getv$.
 \item[(v)] There exists a sequence $\mu_k\in\C$, $k\in\N$ such that for some (equivalently, for every) sequence of  weights $w_k\in(0,1]$, $k\in\N$
such that $\sum_{k\in\N}w_k=1$, the measure $\nu=\sum_{k\in\N} w_k\mu_k$ satisfies $\nu\getv\mu$ for every $\mu\in\C$.
 \item[(vi)]  $\C$ is separable with respect to the total variation distance.
 \item[(vii)]  Let $ \C^+:=\{\mu\in\mathcal P: \exists \rho\in\C\, \rho\getv\mu\}$. Every disjoint (with respect to $\getv$) subset of $\C^+$ is at most countable.
\end{itemize}
Moreover, every solution to any of the Problems 1-3 is a solution to the other two, as is any upper bound for $\C$.
The sequence $\mu_k$ in the statement (v)
 can be taken to be any dense (in the total variation distance) countable subset of $\C$ (cf. (vi)), or any maximal disjoint (with respect to $\getv$) subset of $\C^+$ of statement
(vii), in which every measure that is not in $\C$ is replaced by any measure from $\C$ that dominates it. 
\end{theorem}
\begin{proof}
The implications $(i)\Leftarrow(ii)\Leftarrow(iii)$ are obvious (cf.\ Proposition~\ref{th:1}). 
The implication $(iv)\Rightarrow(i)$ is a reformulation of 
the result of \cite{Blackwell:62}. The converse (and hence $(v)\Rightarrow(iv)$) was established in \citep{Kalai:94}.
$(i)\Rightarrow(ii)$ follows from the equivalence $(i)\Leftrightarrow(iv)$ and the transitivity of $\getv$;
 $(i)\Rightarrow(iii)$ follows from the transitivity of $\getv$  and from Lemma~\ref{th:01} below: indeed, from  Lemma~\ref{th:01} we have $l_{tv}(\nu,\mu)=0$
if $\mu\ge_{tv}\nu$ and $l_{tv}(\nu,\mu)=1$ otherwise. From this and the transitivity of $\ge_{tv}$ it follows that 
 if $\rho\ge_{tv}\mu$ then also $l_{tv}(\nu,\rho)\le l_{tv}(\nu,\mu)$ for all $\nu\in\mathcal P$. 
  The equivalence
of $(v)$, $(vi)$, and $(i)$ was established in \citep{Ryabko:10pq3+}. The equivalence of $(iv)$ and $(vii)$ was proven in \citep{Plesner:46}.
The concluding statements of the theorem are easy to demonstrate from the  results cited above.
\end{proof}

The following lemma is an  easy consequence of  \citep{Blackwell:62}.
\begin{lemma}\label{th:01} Let $\mu, \rho$ be two process measures. Then $v(\mu,\rho,x_{1..n})$ converges
to either 0 or 1 with $\mu$-probability~1. 
\end{lemma}~\ref{th:01}
\begin{proof} Assume that $\mu$ is not absolutely continuous with respect to $\rho$ (the other
case is covered by \citep{Blackwell:62}).  By Lebesgue decomposition theorem, the measure $\mu$ admits a representation $\mu=\alpha \mu_a + (1-\alpha)\mu_s$ where 
$\alpha\in[0,1]$ and the measures  $\mu_a$ and $\mu_s$ are such that $\mu_a$ is absolutely continuous with 
  respect to $\rho$ and $\mu_s$ is singular with respect to $\rho$. 
Let $W$ be such a set that $\mu_a(W)=\rho(W)=1$ and $\mu_s(W)=0$.
  Note that we can take $\mu_a=\mu|_{W}$ and $\mu_s=\mu|_{\X^\infty\backslash W}$.
From \citep{Blackwell:62} we have $v(\mu_a,\rho,x_{1..n})\to0$ $\mu_a$-a.s., as well as $v(\mu_a,\mu,x_{1..n})\to0$
$\mu_a$-a.s.\ and $v(\mu_s,\mu,x_{1..n})\to0$ $\mu_s$-a.s. Moreover, $v(\mu_s,\rho,x_{1..n})\ge |\mu_s(W|x_{1..n})-\rho(W|x_{1..n})|=1$
so that $\v(\mu_s,\rho,x_{1..n})\to1$ $\mu_s$-a.s. Furthermore,
$$
 v(\mu,\rho,x_{1..n})\le v(\mu,\mu_a,x_{1..n})+v(\mu_a,\rho,x_{1..n})=I
$$
and 
$$
 v(\mu,\rho,x_{1..n})\ge - v(\mu,\mu_s,x_{1..n})+v(\mu_s,\rho,x_{1..n})=II.
$$
We have $I\to0$ $\mu_a$-a.s.\ and hence $\mu|_W$-a.s., as well as  $II\to 1$ $\mu_s$-a.s.\ and hence $\mu|_{\X^\infty\backslash W}$-a.s.
Thus, $\mu(v(\mu,\rho,x_{1..n})\to 0\text{ or }1)\le \mu(W)\mu|_W(I\to0)+\mu(\X^\infty\backslash W)\mu|_{\X^\infty\backslash W}(II\to1)=\mu(W)+\mu(\X^\infty\backslash W)=1$,
which concludes the proof. 
\end{proof}

\noindent{\bf Remark.} Using Lemma~\ref{th:01} we can also define {\em expected} (rather than almost sure) total variation 
loss of $\rho$ with respect to $\mu$, as the $\mu$-probability that $v(\mu,\rho)$ converges to~1: 
$$
l'_{tv}(\mu,\rho):=\mu\{x_1,x_2,\dots\in\X^\infty:v(\mu,\rho,x_{1..n})\to1\}.
$$
Then Problem~3 can be reformulated
 for this notion of loss. However, it is easy to see that for this reformulation Theorem~\ref{th:tv} holds true as well.

Thus, we can see that, for the case of prediction in total variation, all the sequence prediction problems formulated reduce to studying 
the relation of absolute continuity for process measures and those families of measures that are absolutely continuous (have a density) with 
respect to some measure (a predictor). 
 On the one hand, from a statistical 
point of view  such families
are rather large: the assumption that the probabilistic law in question has a density with respect to some (nice) measure
is  a standard one in statistics. It should also be mentioned that  such families can easily be uncountable. (In particular, 
this means that they are large from a computational point of view.) 
On the other hand, even such basic examples as the set of all Bernoulli i.i.d.\ measures does not allow for a predictor 
that predicts every measure in total variation (as explained in Section~\ref{s:pre}). 

That is why we have to consider weaker notions of predictions; from these, prediction in expected average KL divergence 
 is perhaps one of the weakest. The goal of the next sections is to see which of the properties
that we have for total variation can be transferred (and in which sense) to the case of expected average KL divergence. 

\section{Prediction in Expected Average KL Divergence}
First of all, we have to observe that  for prediction in KL divergence Problems 1, 2, and 3 are different, as the following theorem shows.
While the examples provided in the proof  are artificial, there is a very important 
example illustrating the difference between Problem~1 and Problem~3 for expected average KL divergence: the set $\S$ of all stationary 
processes, given in Theorem~\ref{th:st} in the end of this section. 
 \begin{theorem}\label{th:comp} 
For the case of prediction in expected average KL divergence, Problems 1, 2 and 3 are different: there exists a set $\C_1\subset\mathcal P$  
 for which there is a solution to Problem 1 but there is no solution to Problem~2, and there is a set $\C_2\subset\mathcal P$ for which there is a solution to Problem 2 but 
 there is no solution to Problem~3.
 \end{theorem}
\begin{proof}
 We have to provide two examples. Fix the binary alphabet $\X=\{0,1\}$.
For each deterministic sequence $t=t_1,t_2,\dots\in\X^\infty$ construct the process measure $\gamma_t$ as follows:
$\gamma_t(x_{n}=t_n|t_{1..n-1}):=1-{1\over n+1}$ and for $x_{1..n-1}\ne t_{1..n-1}$ let $\gamma_t(x_{n}=0|x_{1..n-1})=1/2$, for all $n\in\N$.
That is, $\gamma_t$ is Bernoulli i.i.d.\ 1/2 process measure strongly biased towards a specific deterministic sequence, $t$.
Let also $\gamma(x_{1..n})=2^{-n}$ for all $x_{1..n}\in \X^n$, $n\in\N$ (the Bernoulli i.i.d.\ 1/2).
For the set $\C_1:=\{\gamma_t: t\in \X^\infty\}$ we have a solution to Problem 1: indeed, $d_n(\gamma_t,\gamma)\le 1 =o(n)$.
However, there is no solution to Problem~2. Indeed, for each $t\in \mathcal D$ we have $d_n(t,\gamma_t)= \log n=o(n)$ (that is, 
for every deterministic measure there is an element of $\C_1$ which predicts it), while
by Lemma~\ref{th:disc}  for every $\rho\in\mathcal P$ there exists $t\in\mathcal D$ such that $d_n(t,\rho)\ge n$ for all $n\in\N$ (that is, 
there is no predictor which predicts every measure that is predicted by at least one element of $\C_1$).

The second example is similar. For each deterministic sequence $t=t_1,t_2,\dots\in\mathcal D$ construct the process measure $\gamma_t$ as follows:
$\gamma'_t(x_{n}=t_n|t_{1..n-1}):=2/3$ and for $x_{1..n-1}\ne t_{1..n-1}$ let $\gamma'_t(x_{n}=0|x_{1..n-1})=1/2$, for all $n\in\N$.
It is easy to see that $\gamma$ is a solution to Problem 2 for the set $\C_2:=\{\gamma'_t: t\in \X^\infty\}$.
Indeed, if $\nu\in\mathcal P$ is such that $d_n(\nu,\gamma')=o(n)$ then we must have $\nu(t_{1..n})=o(1)$. From 
this and the fact that $\gamma$ and $\gamma'$ coincide (up to $O(1)$) on all other sequences we conclude $d_n(\nu,\gamma)=o(n)$.
 However, there is no 
solution to Problem 3 for $\C_2$. Indeed, for every $t\in\mathcal D$ we have $d_n(t,\gamma'_t)=n \log3/2+o(n)$. Therefore, if $\rho$ is a solution 
to Problem 3 then $\limsup {1\over n} d_n(t,\rho)\le \log 3/2 <1$ which contradicts Lemma~\ref{th:disc}.
\end{proof}

Thus, prediction in expected average KL divergence turns out to be  a more complicated matter than prediction 
in total variation. The next idea is to try and see which of the facts about prediction in total variation can be generalized
to some of the problems concerning prediction in expected average KL divergence.

First, observe that, for the case of prediction in total variation, the equivalence of Problems~1 and~2 was derived 
from the transitivity of the relation $\getv$ of absolute continuity. For the case of expected average KL divergence, the relation 
``$\rho$ predicts $\mu$ in expected average KL divergence'' is not transitive (and Problems~1 and~2 are not equivalent). 
However, for Problem~2 we are interested in the following relation: $\rho$ ``dominates'' $\mu$ if $\rho$ predicts every $\nu$ such that $\mu$ predicts $\nu$.
 Denote this relation by $\geklz$:
\begin{definition}[$\geklz$]
We write $\rho\geklz\mu$ if for every $\nu\in\mathcal P$ the equality $\limsup{1\over n} d_n(\nu,\mu)=0$ implies $\limsup{1\over n} d_n(\nu,\rho)=0$.
\end{definition}
The relation $\geklz$ has some similarities with $\getv$. First of all, $\geklz$ is also transitive (as can be easily seen  from the definition).
Moreover, similarly to $\getv$, one can show that for any $\mu,\rho$ any strictly convex combination $\alpha\mu+(1-\alpha)\rho$ is
a supremum of $\{\rho,\mu\}$ with respect to~$\geklz$. Next we will obtain a characterization of predictability with respect to $\geklz$ similar
to one of those obtained for $\getv$.

The key observation is the following. If there is a solution to Problem~2 for a set $\C$ then a solution can be obtained as a Bayesian mixture 
over a countable subset of $\C$. For total variation this is the statement $(v)$ of Theorem~\ref{th:tv}.

\begin{theorem}\label{th:2} Let $\C$ be a set of probability measures on $\Omega$. If there is a measure $\rho$ such that $\rho\geklz\mu$ for every $\mu\in\C$ ($\rho$ is 
a solution to Problem~2)
then there is a sequence $\mu_k\in\C$, $k\in\N$, such that $\sum_{k\in\N} w_k\mu_k\geklz\mu$ for  every $\mu\in\C$, where $w_k$ are some positive weights.
\end{theorem}

The proof is deferred to Section~\ref{s:pr}. 
An analogous result for Problem~1 was established in \citep{Ryabko:09pq3}. (The proof of Theorem~\ref{th:2} is based 
on similar ideas, but is more involved.)

 For the case of Problem~3,  we do not have results similar to Theorem~\ref{th:2} (or statement $(v)$ of Theorem~\ref{th:tv}); in fact, we conjecture
that the opposite is true: there exists a (measurable) set $\C$ of measures such that there is a solution to Problem~3 for $\C$, but there is no Bayesian solution to Problem~3, 
meaning that there is no probability distribution on $\C$ (discrete or not) such that the mixture over $\C$ with respect to this distribution is a solution 
to Problem~3 for $\C$. 

 However, we can take a different route and extend another part of Theorem~\ref{th:tv} to obtain a characterization of sets $\C$ for which a solution to Problem~3 exists.

We have seen that, in the case of prediction in total variation, separability with respect to the topology of this distance
is a necessary and sufficient condition for the existence of a solution to Problems 1-3. 
In the case of expected average KL divergence the situation is somewhat different, since, first of all, (asymptotic average) KL divergence 
is not a metric. While one can introduce a topology based on it, 
separability with respect to this topology turns out to be a sufficient but not a necessary condition for the existence of a predictor, 
as is shown in the next theorem.

\begin{definition}\label{def:dinf} Define the distance $d_\infty(\mu_1,\mu_2)$ on process measures as 
follows 
 \begin{equation}\label{eq:dinf}
  d_\infty(\mu_1,\mu_2)=\limsup_{n\to\infty}\sup_{x_{1..n}\in\X^n}\frac{1}{n}\left|\log\frac{\mu_1(x_{1..n})}{\mu_2(x_{1..n})}\right|,
 \end{equation} 
where we assume $\log0/0:=0$.
\end{definition}
Clearly, $d_\infty$ is symmetric and satisfies the triangle inequality, but it is not exact.
Moreover, for every $\mu_1,\mu_2$ we have
\begin{equation}
 \limsup_{n\to\infty}\frac{1}{n}d_n(\mu_1,\mu_2)\le d_\infty(\mu_1,\mu_2).
\end{equation}
The distance $d_\infty(\mu_1,\mu_2)$ measures the difference in behaviour of $\mu_1$ and $\mu_2$ on all individual sequences.
Thus, using this distance to analyse Problem~3 is most close to the traditional approach to the non-realizable case, which is 
formulated in terms of predicting  individual deterministic sequences.
\begin{theorem}\label{th:dinf}  
\begin{itemize}
  \item[(i)] Let $\C$ be a set of process measures. If $\C$ is separable with respect to $d_\infty$ then 
there is a solution to Problem~3 for $\C$, for the case of prediction in expected average KL divergence. 
 \item[(ii)] There exists a set of process measures $\C$ such that $\C$ is not separable with respect to $d_\infty$, but there is a solution to Problem~3 for 
this set, for the case of prediction in expected average KL divergence.
\end{itemize}
\end{theorem}
\begin{proof}
For the first statement,  let $\C$ be separable and let $(\mu_k)_{k\in\N}$ be a dense countable subset of $\C$. Define $\nu:=\sum_{k\in\N}w_k\mu_k$, 
where $w_k$ are any positive summable weights. Fix any measure $\tau$ and any $\mu\in\C$.
We will show that $\limsup_{n\to\infty}{1\over n}d_n(\tau,\nu)\le \limsup_{n\to\infty}{1\over n}d_n(\tau,\mu)$.
For every $\epsilon$, find such a $k\in\N$ that $d_\infty(\mu,\mu_k)\le\epsilon$. We have
\begin{multline*}
d_n(\tau,\nu)\le d_n(\tau,w_k\mu_k) =\E_\tau\log\frac{\tau(x_{1..n})}{\mu_k(x_{1..n})}-\log w_k
\\= \E_\tau\log\frac{\tau(x_{1..n})}{\mu(x_{1..n})} + \E_\tau\log\frac{\mu(x_{1..n})}{\mu_k(x_{1..n})} -\log w_k\\
\le d_n(\tau,\mu)+\sup_{x_{1..n}\in\X^n}\log\left|\frac{\mu(x_{1..n})}{\mu_k(x_{1..n})}\right| -\log w_k.
\end{multline*}
From this, dividing by $n$ taking $\limsup_{n\to\infty}$ on both sides, we conclude 
$$
\limsup_{n\to\infty}{1\over n}d_n(\tau,\nu)\le \limsup_{n\to\infty}{1\over n}d_n(\tau,\mu) + \epsilon.
$$
Since this holds for every $\epsilon>0$ the first statement is proven.

The second statement is proven by the following example. 
 Let $\C$ be the set of all deterministic sequences (measures concentrated on just one sequence) such that the number of 0s in the 
 first $n$ symbols is less than $\sqrt{n}$, for all $n\in\N$. Clearly, this set is uncountable. 
It is easy to check that $\mu_1\ne\mu_2$ implies $d_\infty(\mu_1,\mu_2)=\infty$ for every $\mu_1,\mu_2\in\C$, but 
 the predictor $\nu$, given by $\nu(x_n=0)=1/n$ independently for different $n$, predicts every $\mu\in\C$ in expected average KL divergence. Since all elements
of $\C$ are deterministic, $\nu$ is also a solution to Problem~3 for $\C$.
\end{proof}

Although simple, Theorem~\ref{th:dinf} can be used to establish the existence of a solution to Problem~3 for an important 
class of process measures: that of all processes with  finite memory, as the next theorem shows. Results similar to Theorem~\ref{th:mark} are known 
in different settings, e.g., \citep{Ziv:78, BRyabko:84, Cesa:99} and others. 

\begin{theorem}\label{th:mark}
 There exists a solution to Problem~3 for prediction in expected average KL divergence for the set of all finite-memory process measures $\mathcal M:=\cup_{k\in\N}\mathcal M_k$.
\end{theorem}
\begin{proof}
 We will show that the set $\mathcal M$ is separable with respect to $d_\infty$. Then the statement will follow from Theorem~\ref{th:dinf}.
It is enough to show  that each set $\mathcal M_k$ is separable with respect to $d_\infty$. 

For simplicity, assume that the alphabet is binary ($|\X|=2$; the general case is analogous).
 Observe that the family $\mathcal M_k$ of $k$-order stationary binary-valued Markov
processes is parametrized by $|\X|^{k}$ $[0,1]$-valued parameters: probability of observing $0$ after observing $x_{1..k}$, for each $x_{1..k}\in\X^k$.
Note that this parametrization is continuous (as a mapping from the parameter space with the Euclidean topology to $\mathcal M_k$ with the
topology of $d_\infty$). 
Indeed,   for any $\mu_1,\mu_2\in\mathcal M_k$ 
and every $x_{1..n}\in\X^n$ such that $\mu_i(x_{1..n})\ne 0$, $i=1,2$, it is easy to see  that 
\begin{equation}\label{eq:m3}
{1\over n}\left |\log\frac{\mu_1(x_{1..n})}{\mu_2(x_{1..n})}\right|\le \sup_{x_{1..k+1}}{{1\over k+1} \left|\log\frac{\mu_1(x_{1..k+1})}{\mu_2(x_{1..k+1})}\right|},
\end{equation}
so that the right-hand side of~(\ref{eq:m3}) also upper-bounds $d_\infty(\mu_1,\mu_2)$, implying continuity of the parametrization.

It follows that  the set  $\mu^k_q$, $q\in Q^{|\X|^k}$  of all stationary $k$-order Markov processes with rational 
values of all the parameters ($Q:=\mathbb Q\cap[0,1]$) 
is dense in $\mathcal M_k$, proving the separability of the latter set.
\end{proof}

Another important example is the set of all stationary process measures $\S$. This example also illustrates the difference between the prediction problems 
that we consider. For this set a solution to Problem~1 was given in \citep{BRyabko:88}. 
In contrast, here we show that there is no solution to  Problem~3 for $\S$.
\begin{theorem}\label{th:st}
 There is no solution to Problem~3 for the set of all stationary processes $\S$. 
\end{theorem}
\begin{proof}
 This   proof  is based on the construction similar to the one used in \citep{BRyabko:88} 
to demonstrate impossibility of consistent prediction  of stationary processes without Cesaro averaging. 

Let $m$ be a Markov chain with states $0,1,2,\dots$ and state transitions defined as follows.
From each sate $k\in\N\cup\{0\}$ the chain passes to the state $k+1$ with probability 2/3 and to the state 0 with probability 1/3. 
It is easy to see that this chain possesses a unique stationary distribution on the set of states (see, e.g., \citealp{Shiryaev:96}); taken as the initial distribution 
it defines a stationary ergodic process with  values in $\N\cup\{0\}$. Fix the ternary alphabet $\X=\{a,0,1\}$. 
For each sequence $t=t_1,t_2,\dots\in\{0,1\}^\infty$ define the process $\mu_t$ as follows. It is a deterministic function 
of the chain $m$. If the chain is in the state 0 then the process $\mu_t$ outputs $a$; if the chain $m$ is in the state $k>0$ then 
the process outputs $t_k$. That is,  we have defined a hidden Markov process which in the state 0 of the underlying 
Markov chain always outputs  $a$, while in other states it outputs either $0$ or $1$ according to the sequence~$t$. 

To show that there is no solution to Problem 3 for $\S$, we will show that there is no solution to Problem~3 for the smaller
set $\C:=\{\mu_t: t\in\{0,1\}^\infty\}$. Indeed, for any $t\in\{0,1\}^\infty$ we have $d_n(t,\mu_t)= n\log 3/2 + o(n)$. Then if $\rho$ 
is a solution to Problem~3 for $\C$ we should have $\limsup_{n\to\infty}{1\over n} d_n(t,\rho)\le \log 3/2<1$ for every $t\in\mathcal D$,
 which contradicts Lemma~\ref{th:disc}.
\end{proof}

From the proof of Theorem~\ref{th:st} one can see that, in fact, the statement that is proven is stronger: there is no 
solution to Problem~3 for the set of all functions of stationary ergodic countable-state Markov chains. 
We conjecture that a solution to Problem~2 exists for the latter set, but not for the set of all stationary processes.

As we have seen in the statements above, the set of all deterministic measures  $\mathcal D$ 
plays an important role in the analysis of the predictors in the sense of Problem~3.
Therefore, an interesting question is to characterize those sets $\mathcal C$ of measures 
for which there is a predictor $\rho$ that predicts {\em every individual sequence} at least
as well as any measure from $\C$.  Such a characterization can be obtained in terms of Hausdorff dimension,  using a result of \cite{BRyabko:86},
that shows that Hausdorff dimension of a set characterizes the optimal prediction error that can be attained by any predictor.

For a set $A\subset \X^\infty$ denote $H(A)$ its Hausdorff dimension (see, for example, \citep{Billingsley:65} for its definition).

\begin{theorem}
 Let $\C\subset\P$. The following statements are equivalent.
\begin{itemize}
 \item[(i)] There is a measure $\rho\in\P$ that predicts every individual sequence at least as well as the best measure from $\C$:  
  for every $\mu\in\C$ and  every sequence $x_1,x_2,\dots\in\X^\infty$ we have
  \begin{equation}\label{eq:ha}
   \liminf_{n\to\infty}-{1\over n}\log\rho(x_{1..n})\le \liminf_{n\to\infty}-{1\over n}\log\mu(x_{1..n}).
  \end{equation}
 \item[(ii)] For every $\alpha\in[0,1]$ the Hausdorff dimension of the set of sequences on which the average prediction error 
 of the best measure in $\C$ is not greater than $\alpha$ is bounded by $\alpha/\log|\X|$:
 \begin{equation}\label{eq:ha2}
  H(\{x_1,x_2,\dots\in\X^\infty:\inf_{\mu\in\C}\liminf_{n\to\infty}-{1\over n}\log\mu(x_{1..n})\le\alpha\})\le \alpha/\log|\X|.
 \end{equation}
\end{itemize}
\end{theorem}
\begin{proof}
 The implication $(i)\Rightarrow(ii)$ follows directly from \citep{BRyabko:86} where it is shown that for every 
measure $\rho$ one must have $H(\{x_1,x_2,\dots\in\X^\infty:\liminf_{n\to\infty}-{1\over n}\log\rho(x_{1..n})\le\alpha\})\le \alpha/\log|\X|$.

To show the opposite implication, we again refer to \citep{BRyabko:86}: for every set $A\subset\X^\infty$ there is a measure $\rho_A$ such 
that   
\begin{equation}\label{eq:hh}
\liminf_{n\to\infty}-{1\over n}\log\rho_A(x_{1..n})\le H(A)\log|\X|.
\end{equation}
 For each $\alpha\in[0,1]$ define
$
 A_\alpha:= \{x_1,x_2,\dots\in\X^\infty:\inf_{\mu\in\C}\liminf_{n\to\infty}-{1\over n}\log\mu(x_{1..n})\le \alpha\}).
$ 
By assumption, $H(A_\alpha)\le\alpha/\log|X|$, so that from~(\ref{eq:hh}) for all $x_1,x_2,\dots\in A_\alpha$ we obtain
\begin{equation}\label{eq:hh2}
\liminf_{n\to\infty}-{1\over n}\log\rho_A(x_{1..n})\le \alpha.
\end{equation}
Furthermore, define $\rho:=\sum_{q\in Q}w_q\rho_{A_q}$, 
where $Q=[0,1]\cap\mathbb Q$ is the set of rationals in $[0,1]$ and $(w_q)_{q\in  Q}$ is any sequence of positive reals
 satisfying $\sum_{q\in Q}w_q=1$. For every $\alpha\in[0,1]$ let $q_k\in Q$, $k\in\N$ be such a sequence that $0\le q_k -\alpha\le1/k$. 
 Then, for every $n\in\N$ and every $x_1,x_2,\dots\in A_{q_k}$ we have 
$$
 -{1\over n}\log\rho(x_{1..n})\le -{1\over n}\log\rho_q(x_{1..n}) -\frac{\log w_{q_k}}{n}.
$$
From this and~(\ref{eq:hh2}) we get
$$
\liminf_{n\to\infty}-{1\over n}\log\rho(x_{1..n})\le \liminf_{n\to\infty} \rho_{q_k}(x_{1..n})+1/k\le q_k+1/k.
$$ Since this holds 
for every $k\in\N$, it follows that for all $x_1,x_2,\dots\in\cap_{k\in\N} A_{q_k}=A_\alpha$ we have
$$
 \liminf_{n\to\infty}-{1\over n}\log\rho(x_{1..n})\le \inf_{k\in\N}(q_k +1/k)=\alpha,
$$
which completes the proof of the implication $(ii)\Rightarrow(i)$.
\end{proof}

\section{Discussion}
It has been long realized that the so-called probabilistic and agnostic (adversarial, non-stochastic, deterministic) settings 
of the problem of sequential prediction are strongly related. This  has been  most evident from looking at the solutions
to these problems, which are usually based on the same ideas. Here we have proposed a formulation of the agnostic 
problem as a non-realizable case of the probabilistic problem. While being very close to the traditional one,  this 
 setting allows us to directly compare the two problems. As a  somewhat surprising result, we can see that 
whether the two problems are different depends on the measure of performance chosen: in the case of prediction in total 
variation distance they coincide, while in the case of prediction in expected average KL divergence they are different.
In the latter case, the distinction becomes particularly apparent on the example of stationary processes:
while a solution to the realizable  problem has long been known, here we have shown that there is no solution 
to the agnostic version of this problem. 
This formalization also allowed us to introduce another problem that lies in between the realizable and  the fully agnostic problems: given a class of process measures 
$\C$, find a predictor whose  predictions are  asymptotically correct for  every measure for which at least one of the measures in $\C$ gives
 asymptotically correct predictions (Problem~2).
This problem is less restrictive then the fully agnostic  one  (in particular, it is not concerned 
with the behaviour of a predictor on every deterministic sequence)  but at the same time the solutions to this problem 
have performance guarantees far outside the model class considered.

In essence, the formulation of Problem~2 suggests to assume that we have a set of models one of which is good enough to make predictions, with the goal 
of combining the predictive powers of these models.
This is perhaps a  good compromise between making modelling assumptions on the data (the data is generated by one of the models we have) and
the  fully agnostic, worst-case, setting.

Since the problem formulations presented here are mostly new (at least,  in such a general form), it is not 
surprising that there are many questions left open. A promising route  to obtain new results seems
to be to first analyse the case of prediction in total variation, which amounts to studying 
the relation of absolute continuity and singularity of probability measures,  and then to try and find analogues
in less restrictive (and thus more interesting and difficult) cases of predicting only the next observation, possibly 
with Cesaro averaging. This is the approach that we took in this work. Here it is interesting to find 
properties common to all or most of the prediction problems (in total variation as well as with respect to other measures of the performance), if it 
is at all possible.
For example,  the ``countable Bayes'' property of Theorem~\ref{th:2}, that states that if there is a solution to a given sequence
prediction problem for a set $\C$ then a solution can be obtained as a mixture over a suitable countable subset of $\C$, holds for Problems~1--3 in total
variation, and for Problems~1 and~2 in KL divergence; however we conjecture that it does not hold for the Problem~3 in KL divergence.

It may also be interesting to study algebraic properties of the relation $\geklz$ that arises when studying Problem~2.
We have show that it shares some properties with the relation $\getv$ of absolute continuity. Since the latter characterizes
prediction  in total variation and the former characterizes prediction in KL divergence (in the sense of Problem~2), which is much weaker, 
it would  be interesting to see exactly what properties the two relations share.

Another direction for future research concerns finite-time performance analysis. In this work we have adopted the asymptotic approach to the prediction problem, 
ignoring the behaviour of predictors before asymptotic. While for prediction in total variation it is a natural choice, for other measures of performance, 
including average KL divergence, it is clear that Problems 1-3 admit non-asymptotic formulations. 
It is also interesting  what are  the relations between  performance guarantees that can be obtained in non-asymptotic formulations of Problems~1--3.
 

\section{Proof of Theorem~\ref{th:2}}\label{s:pr}
\begin{proof}
Define the sets $C_\mu$ as the set of all measures 
$\tau\in\P$ such that $\mu$ predicts $\tau$ in expected average KL divergence.
Let $\C^+:=\cup_{\mu\in\C} C_\mu$. For each $\tau\in\C^+$ let $p(\tau)$ be any (fixed) 
$\mu\in\C$ such that $\tau\in C_\mu$. In other words,  $\C^+$ is the set of all measures 
that are predicted by some of the measures in $\C$, and for each measure $\tau$ in $\C^+$ we 
designate one ``parent'' measure $p(\tau)$ from $\C$ such that $p(\tau)$ predicts $\tau$.

Define the weights $w_k:=1/k(k+1)$, for all $k\in\N$. 

\noindent{\em Step 1.} 
 For each $\mu\in \C^+$ let $\delta_n$ be any monotonically increasing function such that $\delta_n(\mu)=o(n)$ and $d_n(\mu,p(\mu))=o(\delta_n(\mu))$.
Define the sets 
\begin{equation}\label{eq:U}
U_\mu^n:=\left\{x_{1..n}\in \X^n: \mu(x_{1..n})\ge{1\over n}\rho(x_{1..n})\right\},
\end{equation} 
\begin{equation}\label{eq:V}
V_\mu^n:=\left\{x_{1..n}\in \X^n: p(\mu)(x_{1..n})\ge2^{-\delta_n(\mu)}\mu(x_{1..n})\right\},
\end{equation} 
and
\begin{equation}\label{eq:t}
T_\mu^n:=U_\mu^n\cap V_\mu^n.
\end{equation} 
We will upper-bound $\mu(T_\mu^n)$.
First, using Markov's inequality, we derive 
\begin{equation}\label{eq:markk}
\mu(\X^n\backslash U_\mu^n) 
 = \mu \left(\frac {\rho(x_{1..n})}{\mu(x_{1..n})} > n\right)\le {1\over n} E_\mu \frac {\rho(x_{1..n})}{\mu(x_{1..n})}={1\over n}.
\end{equation}
Next, observe that for every $n\in\N$ and every set $A\subset \X^n$, using Jensen's inequality we can obtain
\begin{multline}\label{eq:jen}
-\sum_{x_{1..n}\in A}\mu(x_{1..n})\log\frac{\rho(x_{1..n})}{\mu(x_{1..n})}
=  -\mu(A)\sum_{x_{1..n}\in A}{1\over\mu(A)}\mu(x_{1..n})\log\frac{\rho(x_{1..n})}{\mu(x_{1..n})}
\\
\ge -\mu(A)\log{\rho(A)\over\mu(A)} \ge -\mu(A)\log\rho(A) -{1\over2}. 
\end{multline}
Moreover,
\begin{multline*}\label{eq:anoth}
 d_n(\mu,p(\mu))=   -\sum_{x_{1..n}\in\X^n\backslash V_\mu^n }\mu(x_{1..n})\log\frac{p(\mu)(x_{1..n})}{\mu(x_{1..n})}  
\\
-\sum_{x_{1..n}\in V_\mu^n}\mu(x_{1..n})\log\frac{p(\mu)(x_{1..n})}{\mu(x_{1..n})} \ge \delta_n(\mu_n)\mu(\X^n\backslash V_\mu^n)-1/2,
\end{multline*}
where in the inequality we have used~(\ref{eq:V}) for the first summand  and~(\ref{eq:jen}) for the second.
Thus,
\begin{equation}\label{eq:del}
\mu(\X^n\backslash V_\mu^n) \le \frac{d_n(\mu,p(\mu))+1/2}{\delta_n(\mu)}=o(1).
\end{equation}
From~(\ref{eq:t}), (\ref{eq:markk}) and~(\ref{eq:del}) we conclude
\begin{equation}\label{eq:mark}
\mu(\X^n\backslash T_\mu^n) \le \mu(\X^n\backslash V_\mu^n) + \mu(\X^n\backslash U_\mu^n) =o(1).
\end{equation}

{\em Step 2n: a countable cover, time $n$.}
Fix an $n\in\N$. Define $m^n_1:=\max_{\mu\in\C}\rho(T_\mu^n)$ (since $\X^n$ are finite all suprema are reached). 
 Find any $\mu^n_1$ such that $\rho^n_1(T_{\mu^n_1}^n)=m^n_1$ and let
$T^n_1:=T^n_{\mu^n_1}$. For $k>1$, let $m^n_k:=\max_{\mu\in\C}\rho(T_\mu^n\backslash T^n_{k-1})$. If $m^n_k>0$, let $\mu^n_k$ be any $\mu\in\C$ such 
that $\rho(T_{\mu^n_k}^n\backslash T^n_{k-1})=m^n_k$, and let $T^n_k:=T^n_{k-1}\cup T^n_{\mu^n_k}$; otherwise let $T_k^n:=T_{k-1}^n$. Observe that 
(for each $n$) there is only a finite number of positive $m_k^n$,
since the set $\X^n$ is finite; let $K_n$ be the largest index $k$ such that $m_k^n>0$. Let 
\begin{equation}\label{eq:nun}
\nu_n:=\sum_{k=1}^{K_n} w_kp(\mu^n_k).
\end{equation}
As a result of this construction, for every $n\in\N$ every $k\le K_n$ and every  $x_{1..n}\in T^n_k$ 
using the definitions~(\ref{eq:t}), (\ref{eq:U}) and~(\ref{eq:V})  we obtain
\begin{equation}\label{eq:ext}
\nu_n(x_{1..n})\ge w_k{1\over n}2^{-\delta_n(\mu)}\rho(x_{1..n}).
\end{equation}

{\em Step 2: the resulting predictor.}
Finally, define 
\begin{equation}\label{eq:nu}
\nu:={1\over 2}\gamma+{1\over2}\sum_{n\in\N}w_n\nu_n,
\end{equation}
 where $\gamma$ is the i.i.d.\ measure with equal probabilities of all $x\in\X$ 
(that is, $\gamma(x_{1..n})=|\X|^{-n}$ for every $n\in\N$ and every $x_{1..n}\in\X^n$). 
We will show that $\nu$  predicts every $\mu\in\C^+$, and 
then in the end of the proof (Step~r) we will show how to replace $\gamma$ by a combination of a countable set of elements of $\C$ (in fact, $\gamma$ 
is just a regularizer which ensures that $\nu$-probability of any word is never too close to~0). 

{\em Step 3: $\nu$ predicts every $\mu\in\C^+$.}
Fix any $\mu\in\C^+$. 
Introduce the parameters $\epsilon_\mu^n\in(0,1)$, $n\in\N$, to be defined later, and let $j_\mu^n:=1/\epsilon_\mu^n$.
Observe that $\rho(T^n_k\backslash T^n_{k-1})\ge \rho(T^n_{k+1}\backslash T^n_k)$, for any $k>1$ and any $n\in\N$, by definition of these sets.
Since the sets $T^n_k\backslash T^n_{k-1}$, $k\in\N$ are disjoint, we obtain $\rho(T^n_k\backslash T^n_{k-1})\le 1/k$. Hence,  $\rho(T_\mu^n\backslash T_j^n)\le \epsilon_\mu^n$ for some $j\le j_\mu^n$,
since  otherwise $m^n_j=\max_{\mu\in\C}\rho(T_\mu^n\backslash T^n_{j_\mu^n})> \epsilon_\mu^n$ so that  $\rho(T_{j_\mu^n+1}^n\backslash T^n_{j_\mu^n}) > \epsilon_\mu^n=1/j_\mu^n$, which is a contradiction. 
Thus,   
\begin{equation}\label{eq:tm}
\rho(T_\mu^n\backslash T_{j_\mu^n}^n)\le \epsilon_\mu^n.
\end{equation}
We can upper-bound $\mu(T_\mu^n\backslash T^n_{j^n_\mu})$ as follows. 
First, observe that
\begin{multline}\label{eq:mut}
d_n(\mu,\rho) 
=   -\sum_{x_{1..n}\in T^n_\mu\cap T^n_{j^n_\mu}}\mu(x_{1..n})\log\frac{\rho(x_{1..n})}{\mu(x_{1..n})} 
\\
-\sum_{x_{1..n}\in T^n_\mu\backslash  T^n_{j^n_\mu}}\mu(x_{1..n})\log\frac{\rho(x_{1..n})}{\mu(x_{1..n})} \\- \sum_{x_{1..n}\in \X^n\backslash T^n_\mu}\mu(x_{1..n})\log\frac{\rho(x_{1..n})}{\mu(x_{1..n})}
\\
=
I+II+III.
\end{multline}
Then, from~(\ref{eq:t}) and~(\ref{eq:U}) we get 
\begin{equation}\label{eq:e1}
I\ge -\log n.
\end{equation}
From~(\ref{eq:jen}) and~(\ref{eq:tm})
we get 
\begin{equation}\label{eq:e2}
II
\ge  -\mu(T_\mu^n\backslash T^n_{j^n_\mu}) \log\rho(T_\mu^n\backslash T^n_{j^n_\mu})- 1/2
\ge -\mu(T_\mu^n\backslash T^n_{j^n_\mu}) \log \epsilon_\mu^n - 1/2.
\end{equation}
Furthermore,  
\begin{multline}\label{eq:e3}
III\ge \sum_{x_{1..n}\in \X^n\backslash T^n_\mu}\mu(x_{1..n})\log\mu(x_{1..n}) \\
\ge \mu(\X^n\backslash T^n_\mu)\log\frac{\mu(\X^n\backslash T^n_\mu)}{|\X^n\backslash T^n_\mu|}\ge -{1\over2} - \mu(\X^n\backslash T^n_\mu)n\log|\X|,
\end{multline} 
where the first inequality is obvious, in the second inequality we have used the fact that entropy is maximized when all events are equiprobable 
and in the third one we used $|\X^n\backslash T^n_\mu|\le|\X|^n$.
Combining~(\ref{eq:mut}) with the bounds~(\ref{eq:e1}), (\ref{eq:e2}) and~(\ref{eq:e3})  we obtain 
\begin{equation*}
d_n(\mu,\rho) \ge -\log n  -\mu(T_\mu^n\backslash T^n_{j^n_\mu}) \log \epsilon_\mu^n  - 1 - 
  \mu(\X^n\backslash T^n_\mu)n\log|\X|,
\end{equation*}
so that
\begin{equation}\label{eq:mu2}
 \mu(T_\mu^n\backslash T^n_{j^n_\mu}) \le {1\over-\log \epsilon_\mu^n}\Big(d_n(\mu,\rho) +\log n  +1 + 
     \mu(\X^n\backslash T^n_\mu)n\log|\X| \Big).
\end{equation}
From the fact that $d_n(\mu,\rho)=o(n)$ and~(\ref{eq:mark}) it follows that the term in brackets is $o(n)$, so that 
 we can define the parameters $\epsilon^n_\mu$ in such a way that $-\log \epsilon^n_\mu=o(n)$ while
at the same time  the bound~(\ref{eq:mu2}) gives $\mu(T_\mu^n\backslash T^n_{j^n_\mu})=o(1)$. Fix such a choice of $\epsilon^n_\mu$.
Then,   using~(\ref{eq:mark}),  we  conclude
\begin{equation}\label{eq:xt}
\mu(\X^n\backslash T^n_{j^n_\mu})\le \mu(\X^n\backslash T^n_{\mu})+ \mu(T^n_{\mu}\backslash T^n_{j^n_\mu}) =o(1).
\end{equation}

We proceed with the proof of $d_n(\mu,\nu)=o(n)$. 
For any $x_{1..n}\in T^n_{j_\mu^n}$  we have
\begin{equation}\label{eq:i}
\nu(x_{1..n})\ge {1\over 2}w_n\nu_n(x_{1..n})
\ge{1\over 2}w_n w_{j_{\mu}^n} {1\over n}2^{-\delta_n(\mu)}\rho(x_{1..n})
\ge\frac{w_n}{4n}(\epsilon_\mu^n)^22^{-\delta_n(\mu)}\rho(x_{1..n}),
\end{equation}
where the first inequality follows from~(\ref{eq:nu}), the second from~(\ref{eq:ext}), and in the third we have used $w_{j_{\mu}^n}=1/(j_{\mu}^n)(j_{\mu}^n+1)$
and  $j_\mu^n=1/\epsilon^\mu_n$.
 Next we use the decomposition
\begin{equation}\label{eq:12}
d_n(\mu,\nu)= -\sum_{x_{1..n}\in T^n_{j_\mu^n}}\mu(x_{1..n})\log\frac{\nu(x_{1..n})}{\mu(x_{1..n})} 
- \sum_{x_{1..n}\in \X^n\backslash T^n_{j_\mu^n}}\mu(x_{1..n})\log\frac{\nu(x_{1..n})}{\mu(x_{1..n})}  = I + II.
\end{equation}
From~(\ref{eq:i})  we find 
\begin{multline}\label{eq:1}
I\le -\log\left(\frac{w_n}{4n}(\epsilon_\mu^n)^22^{-\delta_n(\mu)}\right)  - \sum_{x_{1..n}\in T^n_{j_\mu^n}}\mu(x_{1..n})\log\frac{\rho(x_{1..n})}{\mu(x_{1..n})}\hfill\\
=
(o(n) - 2\log\epsilon_\mu^n + \delta_n(\mu)) 
 +\left(d_n(\mu,\rho)+ \sum_{x_{1..n}\in\X^n\backslash T^n_{j_\mu^n}}\mu(x_{1..n})\log\frac{\rho(x_{1..n})}{\mu(x_{1..n})}\right)
\\
\le o(n) -  \sum_{x_{1..n}\in\X^n\backslash T^n_{j_\mu^n}}\mu(x_{1..n})\log\mu(x_{1..n})\\ 
\le o(n)+\mu(\X^n\backslash T^n_{j_\mu^n})n\log|\X|=o(n),
\end{multline}
where in the second inequality we have used $-\log\epsilon_\mu^n=o(n)$, $d_n(\mu,\rho)=o(n)$ and $\delta_n(\mu)=o(n)$, in the last inequality we have again used the fact that the entropy is maximized when all events are equiprobable, 
while the last equality follows from~(\ref{eq:xt}). 
Moreover, from~(\ref{eq:nu}) we find
\begin{equation}\label{eq:2}
II\le \log 2 - \sum_{x_{1..n}\in\X^n\backslash  T^n_{j_\mu^n}}\mu(x_{1..n})\log\frac{\gamma(x_{1..n})}{\mu(x_{1..n})}
\le 1 +n\mu(\X^n\backslash T^n_{j_\mu^n})\log|\X|=o(n),
\end{equation}
where in the last inequality we have used $\gamma(x_{1..n})=|\X|^{-n}$ and $\mu(x_{1..n})\le 1$, and the last equality follows from~(\ref{eq:xt}).

From~(\ref{eq:12}), (\ref{eq:1}) and~(\ref{eq:2}) we conclude ${1\over n}d_n(\nu,\mu)\to0$.

{\em Step r: the regularizer $\gamma$}. It remains to show that the  i.i.d.\ regularizer $\gamma$ in the definition of $\nu$~(\ref{eq:nu}), can be replaced by a convex combination of a countably many elements from $\C$.
Indeed, for each $n\in\N$, denote
$$
A_n:=\{x_{1..n}\in \X^n: \exists\mu\in\C\ \mu(x_{1..n})\ne0\},
$$ and let for each $x_{1..n}\in \X^n$ the measure $\mu_{x_{1..n}}$ be any measure from $\C$ such that $\mu_{x_{1..n}}(x_{1..n})\ge{1\over2}\sup_{\mu\in\C}\mu(x_{1..n})$.
Define 
$$
 \gamma_n'(x'_{1..n}):={1\over |A_n|}\sum_{x_{1..n}\in A_n}\mu_{x_{1..n}}(x'_{1..n}),
$$ for each
$x'_{1..n}\in A^n$, $n\in\N$, and let  $\gamma':=\sum_{k\in\N}w_k\gamma'_k$. 
For every $\mu\in\C$ we have 
$$
\gamma'(x_{1..n})\ge w_n|A_n|^{-1} \mu_{x_{1..n}}(x_{1..n})\ge{1\over2} w_n |\X|^{-n} \mu(x_{1..n})
$$ for
every $n\in\N$ and every $x_{1..n}\in A_n$, which clearly suffices to establish the bound $II=o(n)$ as in~(\ref{eq:2}).
\end{proof}

\subsection*{Acknowledgements}
Some of the results have appeared  \citep{Ryabko:10pqout} in the proceedings of  COLT'10.
The author is grateful to the anonymous reviewers for their constructive comments on the paper.
This research was partially supported by the French Ministry of Higher Education and Research, Nord-Pas-de-Calais Regional Council and FEDER through
CPER 2007-2013, ANR projects EXPLO-RA (ANR-08-COSI-004) and Lampada (ANR-09-EMER-007), by 
the European Community's Seventh Framework Programme (FP7/2007-2013) under grant agreement  231495 (project CompLACS), and by
Pascal-2.


\end{document}